%% file: arxiv_version.tex
\documentclass{article}
\pdfoutput=1

\usepackage{times}
\usepackage{amsmath, amssymb}
\usepackage{natbib}
\usepackage[nolist,nohyperlinks]{acronym}
\usepackage{color}
\usepackage{anysize}

\usepackage{bm}
\usepackage{bbm}
\usepackage{bold-extra}

\newcommand{\BlackBox}{\rule{1.5ex}{1.5ex}}  
\newcommand{\qed}{}
\newenvironment{proof}{\par\noindent{\bf Proof\ }}{\hfill\BlackBox\\[2mm]}
 
\newtheorem{theorem}{Theorem}
\newtheorem{lemma}[theorem]{Lemma}

\newtheorem{claim}[theorem]{Claim}

\input{symbol.tex}
\input{acronym.tex}

\marginsize{3cm}{3cm}{3cm}{3cm}

\title{Fast Rates by Transferring from Auxiliary Hypotheses}

\author{
Ilja Kuzborskij \\
Idiap Research Institute \\
Rue Marconi 19, Martigny, Switzerland \\
\texttt{ilja.kuzborskij@idiap.ch} \\
\and
Francesco Orabona \\
Yahoo! Labs \\
229 West 43rd Street, 10036 New York, NY, USA \\
\texttt{francesco@orabona.com} \\
}

\begin{document}

\maketitle
%
%
\begin{abstract}
\input{abstract.tex}
\end{abstract}
%
%
\input{intro.tex}

%
\input{related.tex}
\input{definitions.tex}
\input{problem.tex}
%
\input{guarantees.tex}
%
\input{excess_risk.tex}
\input{implications.tex}

%
\input{general_results.tex}
%
\input{proofs.tex}
\input{conclusions.tex}

%
\subsection*{Acknowledgements} This work is partially supported by the Swiss National Science Foundation Sinergia project Ninapro (I.K.).
\bibliographystyle{plainnat}
\bibliography{learning}
\input{appendix.tex}
%
\end{document}

%% file: symbol.tex
\newcommand{\bA}{\mathbf{A}}
\newcommand{\bx}{\mathbf{x}}

\newcommand{\bI}{\mathbf{I}}
\newcommand{\bu}{\mathbf{u}}

\newcommand{\sY}{\mathcal{Y}}

\newcommand{\bw}{\mathbf{w}}
\newcommand{\bhatw}{\hat{\mathbf{w}}}

\newcommand{\bW}{\mathbf{W}}

\newcommand{\bv}{\mathbf{v}}
\newcommand{\bzero}{\mathbf{0}}

\newcommand{\sD}{\mathcal{D}}
\newcommand{\sL}{\mathcal{L}}

\newcommand{\sX}{\mathcal{X}}

\newcommand{\sF}{\mathcal{F}}
\newcommand{\sW}{\mathcal{W}}
\newcommand{\sV}{\mathcal{V}}
\newcommand{\sH}{\mathcal{H}}
\newcommand{\sG}{\mathcal{G}}

\newcommand{\field}[1]{\mathbb{#1}}

\newcommand{\R}{\field{R}}

\newcommand{\Var}{\mathrm{Var}}

\newcommand{\bvarepsilon}{\bm{\varepsilon}}

\newcommand{\sgn}{\mbox{\sc sgn}}

\newcommand{\scO}{\mathcal{O}}

\newcommand{\reals}{\mathbb{R}}

\newcommand{\bbeta}{\boldsymbol{\beta}}

\newcommand{\tp}{^{\top}}
\newcommand{\ip}[1]{\left\langle #1 \right\rangle}

\newcommand{\MEbr}[2]{\underset{#1}{\mathbb{E}}\left[ #2 \right]}

\newcommand{\Riskh}{\hat{R}_S}
\newcommand{\Risk}{R}
\newcommand{\Rsrc}{R^{\text{src}}}
\newcommand{\Rhat}{\hat{R}}

\newcommand{\Rad}{\mathfrak{R}}
\newcommand{\Radh}{\hat{\mathfrak{R}}_S}
\newcommand{\RadhNoS}[1]{\hat{\mathfrak{R}}_{#1}}

\newcommand{\distas}[1]{\mathbin{\overset{#1}{\kern\z@\sim}}}%
\newcommand{\tail}{\eta}

\DeclareMathOperator*{\esssup}{ess\,sup}
\DeclareMathOperator*{\E}{\mathbb{E}}

\DeclareMathOperator*{\argmin}{argmin}
\DeclareMathOperator*{\Prob}{\mathbb{P}}

\newcommand{\hsrc}{h^{\text{src}}}
\newcommand{\bhsrc}{\mathbf{h^{\text{src}}}}

\newcommand{\htrg}{h}
\newcommand{\trg}{^{\text{trg}}}
\newcommand{\src}{^{\text{src}}}

\newcommand{\regul}{\Omega}
\newenvironment{proof_custom}[1]{\par\noindent{\bf #1\ }}{\qed}
\newenvironment{nameddef}[1]{\\\par\noindent{\bf #1\ }}{\hfill\\}

\newcommand{\algname}{Regularized \ac{ERM}}
\newcommand{\algnamelong}{Regularized \acs{ERM}}

%% file: acronym.tex
\begin{acronym}
\acro{IC}{Improvement Condition}
\acro{RLS}{Regularized Least Squares}
\acro{HTL}{Hypothesis Transfer Learning}
\acro{ERM}{Empirical Risk Minimization}
\acro{RKHS}{Reproducing kernel Hilbert space}
\acro{DA}{Domain Adaptation}
\acro{LOO}{Leave-One-Out}
\acro{HP}{High Probability}
\acro{TL}{Transfer Learning}
\end{acronym}

%% file: abstract.tex
In this work we consider the learning setting where, in addition to the training set, the learner receives a collection of auxiliary hypotheses originating from other tasks.
We focus on a broad class of ERM-based linear algorithms
that can be instantiated with any non-negative smooth loss function and any strongly convex regularizer.
We establish generalization and excess risk bounds, showing that,
if the algorithm is fed with a good combination of source hypotheses, generalization happens at the fast rate $\scO(1/m)$ instead of the usual $\scO(1/\sqrt{m})$.
On the other hand, if the source hypotheses combination is a misfit for the target task, we recover the usual learning rate.
As a byproduct of our study, we also prove a new bound on the Rademacher complexity of the smooth loss class under weaker assumptions compared to previous works.

%% file: intro.tex
\section{Introduction}
In the standard supervised machine learning setting the learner receives a set of labeled examples, known as the training set.
However, very often we have additional information at hand that could be beneficial to the learning process.
One such example is the use of unlabeled data drawn from the marginal distributions, that gives rise to the semi-supervised learning setting~\citep{chapelle2006semi}.
Another example is when the training data is coming from a related problem, as in multi-task learning~\citep{caruana1998multitask}, domain adaptation~\citep{ben2010theory,mansour2009domain_bounds}, and transfer learning~\citep{pan2010survey,taylor2009transfer}.
Among others, there is the use of structural information, such as taxonomy, different views on the same data~\citep{blum1998combining}, or even a sort of privileged information~\citep{vapnik2009new,sharmanska2013learning}.
In the recent years all these directions have received a considerable empirical and theoretical attention.

In this work we focus on a less theoretically studied direction in the use of supplementary information -- learning with \emph{auxiliary hypotheses}, that is classifiers or regressors originating from another tasks.
In particular, in addition to the training set we assume that the learner is supplied with a collection of hypotheses 
and their predictions on the training set itself. 
The goal of the learner is to figure out which hypotheses are helpful and use them to improve the prediction performance of the trained classifier.
We will call these auxiliary hypotheses the \emph{source} hypotheses and we will say that helpful ones accelerate the learning on the \emph{target} task.
We focus on the linear setting, that is, we train a linear\footnote{Non-linear classifiers can be easily produced with the use of kernels.} classifier and the source hypotheses are used additively in the prediction process, weighted by arbitrary weights.
This generalizes the setting in which the outputs of the source hypotheses are concatenated with the feature vector, a widely used heuristic~\citep{classemes, li2010object, tommasi2013learning}.

The scenario described above is related to the \ac{TL} and \ac{DA} ones, or learning effectively from possibly small amount of data by reusing the prior knowledge~\citep{thrun1998learning,pan2010survey,taylor2009transfer,ben2010theory}.
However, transferring from hypotheses offers an advantage compared to \ac{TL} and \ac{DA} frameworks, where one requires access to the data of the \emph{source} domain.
For example, in \ac{DA}~\citep{ben2010theory}, one employs large unlabeled samples to estimate the relatedness of source and target domains to perform the adaptation.
Even if unlabeled data are abundant, the estimation of adaptation parameters can be computationally prohibitive.
This is the case, for example, when a large number of domains is involved or when one acquires new domains incrementally.

A recently proposed setting, closer to the one we consider, is the \ac{HTL} \citep{kuzborskij2013stability,Ben-DavidU13}, where the practical limitations of \ac{TL} and \ac{DA} are alleviated through indirect access to the \emph{source domain} by means of a \emph{source hypothesis}. Also, in the \ac{HTL} setting there are no restrictions on how the source hypotheses can be used to boost the performance on the target task.


Albeit empirically the setting considered in this paper has already been extensively exploited in the past~\citep{yang2007cross,orabona2009model,tommasi2010safety,jie2011multiclass,kuzborskij2013from},
a first theoretical treatment of this setting was given by~\cite{kuzborskij2013stability}, 
where we analyzed the linear \ac{HTL} algorithm that solves a regularized least-squares problem with a single fixed, unweighted, source hypothesis.
We proved a polynomial generalization bound that depends on the performance of the fixed source hypothesis on the target task.

\textbf{Our contributions.}
We extend the formulation in \citep{kuzborskij2013stability}, with a general regularized Empirical Risk Minimization (ERM) problem with respect to any non-negative smooth loss function, not necessarily convex, and any strongly convex regularizer.
We prove high-probability generalization bounds that exhibit \emph{fast rate}, i.e. $\scO(1/m)$, of convergence whenever any \emph{weighted combination} of multiple source hypotheses performs well on the target task.
In addition, we show that, if the combination is perfect, the error on the training set becomes deterministically equal to the generalization error.
Furthermore, we analyze the excess risk of our formulation, and conclude that a good source hypothesis also speeds up the convergence to the performance of the best-in-the-class.
As a byproduct of our study, we prove an upper bound on the Rademacher complexity of a smooth loss class that provides extra information compared to that of Lipschitz loss classes.
Our analysis, that might be of independent interest, is an alternative to the analysis of~\cite{SrebroST10} and it holds under much weaker assumptions.

The rest of the paper is organized as follows.
In the next section we make a brief review of the previous work.
Next, we formally state our formulation in Section~\ref{sec:problem} and present main results right after, in Section~\ref{sec:guarantees}.
In Section~\ref{sec:implications} we discuss the implications and compare them to the body of literature in learning with the fast rates and transfer learning.
Next, in Section~\ref{sec:technical_results}, we present the proofs of our main results. Section~\ref{sec:conc} concludes the paper.

%
%

%% file: related.tex
\section{Related Work}
%
\cite{kuzborskij2013stability} showed that the generalization ability of the regularized least-squares \ac{HTL} algorithm improves if the supplied \emph{source} hypothesis performs well on the target task.
More specifically, we proposed a key criterion, \emph{the risk of the source hypothesis on the target domain}, that captures the relatedness of the source and target domains.
%
Later,~\citet{Ben-DavidU13} showed a similar bound, but with a different quantity capturing the relatedness between source and target.
Instead of considering a general source hypothesis, they have confined their analysis to the linear hypothesis class.
This allowed them to show that the target hypothesis generalizes better when it is close to the good source hypothesis.
From this perspective it is easy to interpret the source hypothesis as an initialization point in the hypothesis class.
Naturally, given a starting position that is close to the best in the class, one generalizes well.

Prior to these works there were few studies trying to understand the learning with auxiliary hypotheses subject to different conditions.
\cite{li2007bayesian} have analyzed a Bayesian approach to \ac{HTL}. 
Employing a PAC-Bayes analysis they showed that given a prior on the hypothesis class, the generalization ability of logistic regression improves if the prior is informative on the target task.
\cite{mansour2009multiple} analyzed a setting of \emph{multiple source hypotheses} combination.
There, in addition to the source hypotheses, the learner receives unlabeled samples drawn from the source distributions, that are used to weight and combine these source hypotheses.
They have studied the possibility of learning in such a scenario, however, they did not address the generalization properties of any particular algorithm.

Unlike these works, we focus on the generalization ability of a large family of \ac{HTL} algorithms, that
generate the target predictor given a set of multiple source hypotheses.
In particular, we analyze Regularized \acl{ERM} with the choice of any non-negative smooth loss and any strongly convex regularizer.
Thus our analysis covers a wide range of algorithms, explaining their empirical success.
One category of those, prevalent in computer vision~\citep{kienzle2006personalized,yang2007cross,tommasi2010safety,aytar2011tabula,kuzborskij2013from,tommasi2013learning}, employs the principle of biased regularization~\citep{scholkopf2001generalized}.
For example, instead of penalizing large weights by introducing the term $\|\bw\|^2$ into the objective function, one enforces them to be close to some ``prior'' model, that is $\|\bw - \bw^{\text{prior}}\|^2$.
This principle also found its applications in other fields, such as NLP~\citep{daume2007frustratingly,daume2010frustratingly}, and electromyography classification~\citep{orabona2009model,tommasi2013improving}.
Many empirical works have also investigated the use of the source hypotheses in a ``black box'' sense, sometimes not even posing the problem as a transfer learning~\citep{duan2009domain,li2010object,jie2011multiclass,classemes},
and recently in conjunction with deep neural networks~\citep{oquab2014learning}.

In the literature there are several other machine learning directions conceptually similar to the one we consider in this work.
Arguably, the most well known one is the \acf{DA} problem.
The standard machine learning assumption is that the training and the testing sets are sampled from the same probability distribution.
In such case, we expect that a hypothesis generated by the learner from that training set will lead to sensible predictions on the testing set.
The difficulty arises when training and testing distributions differ, that is we have a training set sampled from the \emph{source domain} and testing set from the \emph{target domain}.
Clearly, the hypothesis generated from the source domain can perform arbitrarily badly on the target one.
A paradigm of \ac{DA}, addressing this issue has received a lot of attention in recent years~\citep{ben2010theory,mansour2009domain_bounds}.
Although, this framework is different from the one we study in this work, we identify similarities and compare our findings with the theory of learning from different domains in Section~\ref{sec:implications_da}.

%% file: definitions.tex
\section{Definitions}
\label{sec:definition}

In this section we introduce the definitions used in the rest of the paper.

We denote random variables by capital letters.
The expected value of a random variable distributed according to a probability
distribution $\sD$ is denoted by $\E_{X \sim \sD}[X]$ and the variance is denoted by $\Var_{X \sim \sD}[X]$.
The small and capital bold letters will stand respectively for the vectors and matrices, e.g. $\bx = [x_1, \ldots, x_d]\tp$ and $\bA \in \R^{d_1 \times d_2 }~$.
%
%
%

Denoting by $\sX$ and $\sY$ respectively the input and output space of the learning problem,
the training set is $S=\{(\bx_i,y_i)\}_{i=1}^m$, drawn i.i.d. from the probability distribution $\sD$ defined over $\sX \times \sY$.
Without the loss of generality we will have $\sX = \{\bx : \|\bx\| \leq 1\}$ and we will focus on the problems where $\sY = [-C, C]$.

To measure the accuracy of a learning algorithm, we introduce a non-negative loss function $\ell(h(\bx), y)$, which measures the cost incurred predicting $h(\bx)$ instead of $y$.
The \emph{risk} of a hypothesis $h$, with respect to a probability distribution $\sD$,
and the \emph{empirical risk} measured on the sample $S$
are then defined as
\[
\Risk(h) := \E_{(\bx,y) \sim \sD}[\ell(h(\bx), y)],\quad \text{ and } \quad\Riskh(h) := \frac{1}{m} \sum_{i=1}^m \ell(h(\bx_i), y_i).
\]
In the following, the risk is measured with respect to the probability distribution of the \emph{target} domain, unless stated otherwise.
We capture the smoothness of the loss function via following definition.
\begin{nameddef}{$H$-smooth loss function.}
We say that a non-negative loss function $\ell : \sY \times \sY \mapsto \reals_+$ is $H$\emph{-smooth} iff,
\[
\forall t,r \in \reals, \forall y \in \sY, ~ |\nabla_t \ell(t, y) - \nabla_r \ell(r, y)| \leq H |t - r|.
\]
\end{nameddef}
In this work we will make use of strongly convex regularizers, functions that are defined as follows.
\begin{nameddef}{Strongly convex function.}
A function $\regul$ is $\sigma$-strongly convex w.r.t. a norm $\|\cdot\|$ iff for all $\bw, \bv$, and $\alpha \in (0, 1)$ we have
\[
\regul(\alpha \bw + (1 - \alpha) \bv) \leq \alpha \regul(\bw) + (1-\alpha) \regul(\bv) - \frac{\sigma}{2} \alpha (1-\alpha) \|\bw - \bv\|^2.
\]
\end{nameddef}
%
We will quantify the complexity of a hypothesis class by the means of Rademacher complexity~\citep{BartlettM03}.
In particular, the empirical Rademacher complexity of the hypothesis class $\sH$ measured on the sample $S$
and its expectation are defined as
\[
\Radh(\sH) := \E_{\bvarepsilon} \left[ \sup_{h \in \sH} \frac{1}{m} \sum_{i=1}^m \varepsilon_i h(\bx_i) \right] \quad \text{ and } \quad
\Rad(\sH) := \E_S\left[ \Radh(\sH) \right].
\]
Here, $\varepsilon_i$ is a random variable such that $\mathbb{P}(\varepsilon_i=1) = \mathbb{P}(\varepsilon_i=-1) = \frac{1}{2}$.
%
Similarly, as in the case of the risk, the Rademacher complexity is measured with respect to the probability distribution of the target domain, unless stated otherwise.

%% file: problem.tex
\section{Transferring from Auxiliary Hypotheses}
\label{sec:problem}
%
In the following we will capture and generalize many transfer learning formulations that employ a collection of given \emph{source hypotheses} $\{\hsrc_i : \sX \mapsto \sY\}_{i=1}^n$ within the framework of Regularized \acf{ERM}.
These problems typically involve a criterion for source hypothesis selection and combination with the goal to increase performance on the \emph{target task}~\citep{yang2007cross,tommasi2013learning,kuzborskij2015transfer}.
Indeed, some source hypotheses might come from tasks similar to the target task and the goal of an
algorithm is to select only relevant ones.
In this work we will consider source combination
\begin{equation*}
  \hsrc_{\bbeta}(\bx) := \sum_{i=1}^n \beta_i \hsrc_i(\bx),
\end{equation*}
and target hypothesis
\begin{equation}
  \label{eq:htl_predictor}
  h_{\bw, \bbeta}(\bx) : = \ip{\bw,\bx} + \hsrc_{\bbeta}(\bx),
\end{equation}
with the relevance of the sources characterized by the parameter $\bbeta \in \reals^n$.
We will focus on the Regularized \ac{ERM} formulations
with the choice of any non-negative smooth loss function and any strongly-convex regularizer.
This puts our problem into the class of the ones that can be solved efficiently, yet endowed with interesting properties.
\begin{nameddef}{Regularized ERM for Transferring from Auxiliary Hypotheses.}
  \label{alg:htl_erm}
  Let $\ell : \sY \times \sY \mapsto \reals_+$ be an $H$-smooth loss function and
  let $\regul : \sH \mapsto \reals_+$ be a $\sigma$-strongly convex function w.r.t. a norm $\|\cdot\|$.
  Given the target training set $S = \{(\bx_i, y_i)\}_{i=1}^m$, $\lambda \in \reals_+$, source hypotheses $\{\hsrc_i\}_{i=1}^n$,
  and parameters $\bbeta$ obeying $\Omega(\bbeta) \leq \rho$,
  the algorithm generates the \emph{target hypothesis} $\htrg_{\bhatw, \bbeta}$, such that
  \begin{equation}
    \label{eq:htl_erm}
    \bhatw = \argmin_{\bw \in \sH}\left\{\frac{1}{m} \sum_{i=1}^m \ell\left(\ip{\bw, \bx_i} + \hsrc_{\bbeta}(\bx), y_i\right) + \lambda \regul(\bw)\right\}.
  \end{equation}
\end{nameddef}
Note that~\eqref{eq:htl_erm} is minimized only w.r.t. $\bw$, that is, we do not analyze any particular algorithm that searches for the optimal weights of the source hypotheses.
However, we assume that $\Omega(\bbeta) \leq \rho$, that is we constrain $\bbeta$ through a strongly convex function.
Thus, we cover regularized algorithms generating $\bbeta$, which includes most of the empirical work in this field, and potential new algorithms.

In the following we will pay special attention to a quantity that captures the performance of the source hypothesis combination $\hsrc_{\bbeta}(\bx)$ on the target domain
\[
\Rsrc := R(\hsrc_{\bbeta}).
\]
Our analysis will focus on the generalization properties of $\htrg_{\bhatw, \bbeta}$.
In particular, our main goal will be to understand the impact of the source hypothesis combination on the performance of the target hypothesis.
In our analysis we will discuss various regimes of interest, for example considering the perfect and arbitrarily bad source hypothesis.
Our discussion will touch scenarios where the auxiliary hypotheses accelerate the learning and the conditions when we can provably expect perfect generalization.
Finally, we will consider the consistency of the algorithm \eqref{eq:htl_predictor} and pinpoint conditions when we achieve faster convergence to the performance of the best-in-the-class.

%
One special example covered by our analysis, commonly applied in transfer learning, is the \emph{biased regularization}~\citep{scholkopf2001generalized}.
Consider the following least-squares based algorithm.
\begin{nameddef}{Least-Squares with Biased Regularization.}
  \label{alg:htl_erm_biased_regul}
  Given the target training set $S = \{(\bx_i, y_i)\}_{i=1}^m$, source hypotheses $\{\bw\src_i\}_{i=1}^n \subset \sH$, parameters $\bbeta \in \reals^n$ and $\lambda \in \reals_+$,
  the algorithm generates the target hypothesis $h(\bx) = \ip{\bhatw, \bx}$, where
  \begin{equation}
    \label{eq:erm_biased_regul}
    \bhatw = \argmin_{\bw \in \sH}\left\{\frac{1}{m} \sum_{i=1}^m \left(\ip{\bw, \bx_i} - y_i\right)^2 + \lambda \left\|\bw - \bW\src\bbeta \right\|_2^2\right\}.
  \end{equation}
\end{nameddef}
This problem has a simple intuitive interpretation: minimize the training error on the target training set while keeping the solution close to the linear combination of the source hypotheses.
One can naturally arrive at~\eqref{eq:erm_biased_regul} from a probabilistic perspective:
The solution $\bhatw$ is a maximum a posteriori estimate when the conditional distribution is Gaussian and the prior is a $\bW\src \bbeta$-mean, $\frac{1}{\lambda} \bI$-covariance Gaussian distribution.
%
Even though biased regularization is a simple idea, it found success in a plethora of transfer learning applications,
ranging from computer vision~\citep{kienzle2006personalized,yang2007cross,tommasi2010safety,aytar2011tabula,kuzborskij2013from,tommasi2013learning} to NLP~\citep{daume2007frustratingly}, to electromyography classification~\citep{orabona2009model,tommasi2013improving}.

\begin{claim}
Least-Squares with Biased Regularization is a special case of the Regularized \ac{ERM} in \eqref{eq:htl_predictor}.
\end{claim}
\begin{proof}
  Introduce $\bw'$, such that $\bw' = \bw - \bW\src\bbeta$.
  Then we have that problem~\eqref{eq:erm_biased_regul} is equivalent to
  \begin{align*}
    \min_{\bw \in \sH}\left\{\frac{1}{m} \sum_{i=1}^m \left(\ip{\bw' + \bW\src\bbeta, \bx_i} - y_i\right)^2 + \lambda \left\|\bw' \right\|_2^2\right\},
  \end{align*}
  that in turn is a special version of~\eqref{eq:htl_erm} when $\hsrc_i(\bx) = \ip{\bw\src_i, \bx}$, we use the square loss, and $\|\cdot\|_2^2$ as regularizer.
\qed\end{proof}

Albeit practically appealing, the formulation~\eqref{eq:erm_biased_regul} is limited in the fact that the source hypotheses must be a linear predictor living in the same space of the target predictor.
Instead, the formulation in \eqref{eq:htl_predictor} naturally generalizes the biased regularization formulation, allowing 
to treat the source hypothesis as ``black box'' predictors.
%
%
%
%
%

%% file: guarantees.tex
\section{Main Results}
\label{sec:guarantees}
In this section, we present the main results of this work: generalization and excess risk bounds for the \algnamelong.
In the next section we discuss in detail the implications of these results, while
we defer the proofs to the subsequent sections.

The first bound demonstrates the utility of the perfect combination of source hypotheses, while the second lets us observe the dependency on the arbitrary combination.
In particular, the first bound explicitates the intuition that given the perfect source hypothesis learning is not required.
%
In other words, when $\Rsrc=0$ we have that the empirical risk becomes equal to the risk with probability one.
%
%
%
\begin{theorem}
\label{thm:htl_gen_bound}
Let $\htrg_{\bhatw, \bbeta}$ be generated by \algname, given the $m$-sized training set $S$ sampled i.i.d. from the target domain, source hypotheses $\{\hsrc_i : \|\hsrc_i\|_\infty \leq 1 \}_{i=1}^n$, any source weights $\bbeta$ obeying $\regul(\bbeta) \leq \rho$, and $\lambda \in \reals_+$.
Assume that $\ell(\htrg_{\bhatw, \bbeta}(\bx), y) \leq M$ for any $(\bx, y)$ and any training set.
Then, denoting $\kappa = \frac{H}{\sigma}$ and assuming that $\lambda \leq \kappa$,
we have with probability at least $1 - e^{-\tail}, \ \forall \tail \geq 0$
\begin{align}
  \Risk(\htrg_{\bhatw, \bbeta}) &\leq \Riskh(\htrg_{\bhatw, \bbeta}) + \scO\left( \frac{\Rsrc \kappa}{\sqrt{m} \lambda} + \sqrt{\frac{\Rsrc \rho \kappa^2}{m \lambda}} + \frac{M \tail}{m \log\left(1 + \sqrt{\frac{M \tail}{u\src}}\right)} \right) \label{eq:gen_bound_vanishing} \\
  &\leq \Riskh(\htrg_{\bhatw, \bbeta}) + \scO\left( \frac{\kappa}{\sqrt{m}} \left( \frac{\Rsrc}{\lambda} + \sqrt{\frac{\Rsrc \rho}{\lambda}} \right) + \frac{\kappa}{m} \left( \frac{\sqrt{\Rsrc M \tail}}{\lambda} + \sqrt{\frac{\rho}{\lambda}} \right) \right), \label{eq:gen_bound_fast_rate}
\end{align}
where $u\src = \Rsrc \left( m + \frac{\kappa \sqrt{m}}{\lambda} \right) + \kappa \sqrt{\frac{\Rsrc m \rho}{\lambda}}$.
\end{theorem}

%% file: excess_risk.tex
%
Now we focus on the consistency of the \ac{HTL}.
Specifically, we show an upper bound on the excess risk of the \algname,
which depends on $\Rsrc$, that is the risk of the combined source hypothesis $\hsrc_{\bbeta}$ on the target domain.
We observe that for a small $\Rsrc$, the excess risk shrinks at a fast rate of $\scO(1/m)$.
In other words, a good prior knowledge guarantees not only good generalization, but also the fast recovery of the performance of the best hypothesis in the class.

This bound is similar in spirit to the results of localized complexities, as in works of~\citet{bartlett2005local,SrebroST10}, however we focus on the linear \ac{HTL} scenario rather than a generic learning setting.
Later, in Section~\ref{sec:implications}, we compare our bounds to these works and show that our analysis achieves superior results.

\begin{theorem}
  \label{thm:excess_risk}
  Let $\htrg_{\bhatw, \bbeta}$ be generated by \algname, given the $m$-sized training set $S$ sampled i.i.d. from the target domain, source hypotheses $\{\hsrc_i : \|\hsrc_i\|_\infty \leq 1\}_{i=1}^n$, any source weights $\bbeta$ obeying $\regul(\bbeta) \leq \rho$, and $\lambda \in \reals_+$.
  Then, denoting $\kappa = \frac{H}{\sigma}$, assuming that $\lambda \leq \kappa \leq 1$,
  and setting the regularization parameter
  \begin{equation*}
    \lambda = \scO\left( \sqrt{\frac{\kappa}{\tau} \frac{\Rsrc + \sqrt{\Rsrc \rho}}{\sqrt{m}}+ \frac{\sqrt{\kappa}}{\tau}  \sqrt{\frac{\Rsrc + \sqrt{\Rsrc \rho}}{m^{1.5}}} } \right),
  \end{equation*}
  for any choice of $\tau \geq 0$, we have with high probability that
  \begin{align*}
    &\Risk(\htrg_{\bhatw, \bbeta}) - \min_{\regul(\bw) \leq \tau}\Risk(h_{\bw, \bbeta})\\
&\quad= \scO\left( \frac{\sqrt{\Rsrc} + \sqrt[4]{\Rsrc \rho}}{\sqrt[4]{m}} \sqrt{\kappa \tau} + \frac{\sqrt[4]{\Rsrc} + \sqrt[8]{\Rsrc \rho}}{\sqrt[4]{m^{1.5}}} \sqrt[4]{\kappa \tau^2} + \sqrt{\frac{\Rsrc}{m}} + \frac{1}{m} \right).
  \end{align*}
\end{theorem}

%% file: implications.tex
\subsection{Implications}
\label{sec:implications}
%

%
We start by discussing the effect on the generalization ability of the source hypothesis combination.
Intuitively, a good source hypothesis combination should facilitate transfer learning, while a reasonable algorithm must not fail if we provide it with the bad one.
That said, a natural question to ask here is, what makes a good or bad source hypothesis?
As in previous works in transfer learning and domain adaptation, we capture this notion via a quantity that has two-fold interpretation: (1) the performance of the source hypothesis combination on the target domain; (2) relatedness of source and target domains.
In the theorems presented in the previous sections we denoted it by $\Rsrc$, that is the risk of the source hypothesis combination on the target domain.
In this section we will consider various regimes of interest with respect to $\Rsrc$.

\textbf{\emph{When the source is a bad fit.}} First consider the case when the source hypothesis combination $\hsrc_{\bbeta}$ is useless for the purpose of transfer learning, for example, $\hsrc_{\bbeta}(\bx) = 0$ for all $\bx$.
This corresponds to learning with no auxiliary information.
Then we can assume that $\Rsrc \leq M$, and from Theorem~\ref{thm:htl_gen_bound} we obtain
$
\Risk(\htrg_{\bhatw}) - \Riskh(\htrg_{\bhatw}) \leq \scO\left( 1/ (\sqrt{m} \lambda) \right)
$.
This rate matches the one in the analysis of regularized least-squares~\citep{devito2005model,BousquetE02},
that is a special case of the smooth loss function that the \algname~employs.
On the other hand, \cite{SrebroST10} showed a better worst-case rate $\scO(1/\sqrt{m \lambda})$.
However, their framework builds upon a worst case Rademacher complexity which does not involve the expectation over the sample and does not lead to the dependency on $\Rsrc$ we have obtained in Theorem~\ref{thm:htl_gen_bound}.
We will discuss this problem in details later.

\textbf{\emph{When the source is a good fit.}}
Here we would like to consider the behavior of the algorithm in the finite-sample and asymptotic scenarios.
We first look at the regime of small $m$, in particular $m = \scO(1/\Rsrc)$. In this case, the fast rate term will dominate the bound, and we obtain the convergence rate of $\scO( \sqrt{\rho} / (m \sqrt{\lambda}) )$.
In other words, we can expect a faster convergence when $m$ is small, where ``small'' depends on $\Rsrc$, the quality of combined source hypotheses.
Now consider the asymptotic behavior of the algorithm, particularly when $m$ goes to infinity.
In such case, the algorithm exhibits a rate of $\scO\left(\Rsrc / \sqrt{m} \lambda + \sqrt{(\Rsrc \rho) / m \lambda}\right)$, so $\Rsrc$ controls the constant factor of the rate.
Hence, the quantity $\Rsrc$
governs the transient regime for small $m$ and the asymptotic behavior of the algorithm, predicting a faster convergence in both regimes when it is small.


\textbf{\emph{When source is a perfect fit.}}
It is conceivable that the source hypothesis exploited is the perfect one, that is $\Rsrc = 0$.
In other words, the source hypothesis combination is a perfect predictor for the target domain.
Theorem~\ref{thm:htl_gen_bound} implies that $\Risk(\htrg_{\bhatw, \bbeta}) = \Riskh(\htrg_{\bhatw, \bbeta})$ with probability one.
We note that for many practically used smooth losses, such as square loss, this setting is only realistic if source and target domains match and the problem is noise-free.
However, we can observe $\Rsrc = 0$, for example, when the squared hinge loss, $\ell(z,y) = \max\{0, 1 - zy\}^2$, is used and all target domain examples are classified correctly by the source hypothesis combination, case that is not unthinkable for related domains.

\textbf{\emph{Fast rates.}} There is a number of works in the literature investigating a rate of convergence faster than $1/\sqrt{m}$ subject to different conditions.
In particular, the localized Rademacher complexity bounds of~\citet{bartlett2005local} and~\citet{bousquet2002concentration} can be used to obtain results similar to the second inequality of Theorem~\ref{thm:htl_gen_bound}.
Indeed, Theorem~\ref{the:rad_gen_bound} shows a bound which is very similar to the localized ones, albeit with two differences.
The r.h.s. of the first inequality in Theorem~\ref{the:rad_gen_bound} vanishes when the loss class has zero variance.
Though intuitively trivial, this allows to prove a considerable result in the theory of transfer learning as it quantifies the intuition that no learning is necessary if the source has perfect performance on the target task.
Second, by applying the standard localized Rademacher complexity bounds of~\citet{bousquet2002concentration}, and assuming the use of the Lipschitz loss function, we do not achieve a fast rate of convergence, as can be seen from Theorem~\ref{thm:htl_via_localized_bounds}, shown in the Appendix.
We suspect that assuming the smoothness of the loss function is crucial to prove fast rates in our formulation.

Fast rates for \ac{ERM} with the smooth loss have been thoroughly analyzed by~\citet{SrebroST10}.
Yet, the analysis of our \ac{HTL} algorithm within their framework would yield a bound that is inferior to ours in two respects.
The first concerns the scenario when the combined source hypothesis is perfect, that is $\Rsrc = 0$.
The generalization bound of~\citet{SrebroST10} does not offer a way to show that the empirical risk converges to the risk with probability one -- instead one can only hope to get a fast rate of convergence.
The second problem is in the fact that such bound would depend on the empirical performance of combined source hypothesis.
As we have noted before, the quantity $\Rsrc$ is essential because it captures the degree of relatedness between two domains.
In their bounds, one cannot obtain this relationship through the Rademacher complexity term as we did in our analysis.
The reason for this is the stronger notion of Rademacher complexity that is employed by that framework, involving a supremum over the sample instead of an expectation.
The expectation over the sample of the target distribution is crucial here,
because it allows us to quantify how well the source domain is aligned with the target domain, through the source hypothesis acting as a link.
However, one can attempt to obtain the bound on the empirical risk in terms of $\Rsrc$.
We prove such a bound in the Appendix, Theorem~\ref{thm:srebro_gen_bound},
and conclude that if one has a good source hypothesis or even a perfect one, the rate is $\scO(1/\sqrt[4]{m^3})$, which is worse than ours.
\input{implications_da.tex}
%
\input{implications_minbound.tex}

%% file: implications_da.tex
\subsection{Comparison to Theories of Domain Adaptation and Transfer Learning}
\label{sec:implications_da}
The setting in \ac{DA} is different from the one we study, however, we will briefly discuss the theoretical relationship between the two.
Typically in \ac{DA}, one trains a hypothesis from an altered source training set, striving to achieve good performance on the target domain.
The key question here is how to alter, or to \emph{adapt}, the source training set.
To answer this question, \ac{DA} literature introduces the notion of domain relatedness,
which quantifies the dissimilarities between the marginal distributions of corresponding domains.
Practically, in some cases the domain relatedness can be estimated through a large set of unlabeled samples drawn from both source and target domains.
Theories of \ac{DA}~\citep{ben2010theory,mansour2009domain_bounds,ben2012hardness,mansour2009multiple,cortes2014domain} have proposed a number of such domain relatedness criteria.
Perhaps the most well known are the $d_{\sH \Delta \sH}$-divergence~\citep{ben2010theory}
and its more general counterpart, the Discrepancy Distance~\citep{mansour2009domain_bounds}.
%
%
Typically, this divergence is explicitated in the generalization bound along with other terms controlling the generalization on the target domain.
Let $\Risk_{\sD\trg}(h)$ and $\Risk_{\sD\src}(h)$ denote the risks of the hypothesis $h$, measured w.r.t. the target and source distributions.
Then a well-known result of~\cite{ben2010theory} suggests that for all $h \in \sH$
\begin{equation}
\label{eq:da_bound}
\Risk_{\sD\trg}(h) \leq \Risk_{\sD\src}(h) + \frac{1}{2} d_{\sH \Delta \sH}(\sD\src,\sD\trg) + \varepsilon_{\sH}^\star,
\end{equation}
where $\varepsilon_{\sH}^\star = \min_{h \in \sH}\left\{ \Risk_{\sD\trg}(h) + \Risk_{\sD\src}(h) \right\}$.
This result implies that adaptation is possible
given that $d_{\sH \Delta \sH}(\sD\src,\sD\trg)$ and $\varepsilon^\star$ are small.
One can try to reduce those by controlling the complexity of the class $\sH$ and by minimizing the divergence $d_{\sH \Delta \sH}(\sD\src,\sD\trg)$.
In practice, the latter can be manipulated through an empirical counterpart on the basis of unlabeled samples.
Increasing the complexity of $\sH$ indeed reduces $\varepsilon^\star$, but inflates $d_{\sH \Delta \sH}(\sD\src,\sD\trg)$.
On the other hand, by minimizing $d_{\sH \Delta \sH}(\sD\src,\sD\trg)$ alone puts us under the risk of increasing $\varepsilon^\star$, since the empirical divergence is reduced without taking the labelling into account.

Clearly, this bound cannot be directly compared to our result, Theorem~\ref{thm:htl_gen_bound}.
However, we note the term $\Rsrc$ appearing in our results, which plays a role very similar to $d_{\sH \Delta \sH}$ in~\eqref{eq:da_bound}.
In fact, by defining $\sH = \{\bx \mapsto \ip{\bbeta, \bhsrc(\bx)} \ : \ \regul(\bbeta) \leq \tau\}$, where $\bhsrc(\bx) = [\hsrc_1(\bx), \ldots, \hsrc_n(\bx)]\tp$,
and fixing $h = \hsrc_{\bbeta} \in \sH$ in~\eqref{eq:da_bound}, we can write
\begin{align*}
\Rsrc = \Risk_{\sD\trg}(\hsrc_{\bbeta}) &\leq \Risk_{\sD\src}(\hsrc_{\bbeta}) + d_{\sH \Delta \sH}(\sD\src,\sD\trg) + \varepsilon_{\sH}^\star.
\end{align*}
%
Plugging this into the generalization bound~\eqref{eq:gen_bound_fast_rate} and assuming that $\lambda \leq 1$ and $\rho \leq 1/\lambda$ we have for the target hypothesis $\htrg$ that
\begin{equation}
\label{eq:htl_hdh}
\Risk_{\sD\trg}(\htrg) \leq \Riskh(\htrg) + \scO\left( \frac{\Risk_{\sD\src}(\hsrc_{\bbeta}) + d_{\sH \Delta \sH}(\sD\src,\sD\trg) + \varepsilon_{\sH}^\star}{\sqrt{m} \lambda} + \frac{1}{m \lambda} \right).
\end{equation}
Albeit this inequality shows the generalization ability of the transfer learning algorithm,
comparing to~\eqref{eq:da_bound}, we observe that \ac{DA} and our result agree
on the fact that the divergence between the domains has to be small to generalize well.
In fact, in the formulation we consider, the divergence is controlled in two ways: implicitly, by the choice of $\bhsrc$ and through the complexity of class $\sH$, that is by choosing $\tau$.
%
Second, in~\ac{DA} we expect that a hypothesis performs well on the target only if it performs well on the source.
In our results, this requirement is relaxed.
As a side note, we observe that~\eqref{eq:htl_hdh} captures an intuitive notion that a good source hypothesis has to perform well on its own domain.
Finally, in the theory of \ac{DA} $\epsilon_\sH^\star$ is assumed to be small.
Indeed, if $\epsilon_\sH^\star$ is large, there is no hypothesis that is able to perform well on both domains simultaneously,
and therefore adaptation is hopeless.
In our case, the algorithm can still generalize even with large $\epsilon_\sH^\star$, however this is due to the supervised nature of the framework.

We now turn our attention to the previous theoretical works studying \ac{HTL}-related settings.
Few papers have addressed the theory of transfer learning, where the only information passed from the source domain is the classifier or regressor.
%
%
%
\citet{mansour2009multiple} have addressed the problem of multiple source hypotheses combination, however,
in a different setting.
Specifically, in addition to the source hypotheses, the learner receives the unlabeled samples drawn from the source distributions, that ared used to weight and combine these source hypotheses.
The authors have presented a general theory of such a scenario and did not study the generalization properties of any particular algorithm.
The first analysis of the generalization ability of \ac{HTL} in the similar context we consider here was done by~\citet{kuzborskij2013stability}.
The work focused on the $L2$-regularized least squares and the generalization bound involving the leave-one-out risk instead of the empirical one.
The following result, obtained through an algorithmic stability argument~\citep{BousquetE02}, holds with probability at least $1 - \delta$
\begin{equation}
\label{eq:loo_htl_bound}
\Risk(\htrg) \leq \Riskh^{\text{loo}}(\htrg) + \scO\left( \frac{\sqrt[4]{\Rsrc}}{\sqrt{m \delta} \lambda^{0.75}} \right),
\end{equation}
where $\Rsrc$ is the risk of a single fixed source hypothesis and $\htrg$ is the solution of a Regularized Least Square problem.
We first observe that the shape of the bound is similar to the one obtained in this work, although with the number of differences.
First, contrary to our presented bounds, their bound assumes the use of a fixed source hypothesis, that is not even weighted by any coefficient.
In practice, this is a very strong assumption, as one can receive an arbitrarily bad source and have no way to exclude it.
Second, the bound~\eqref{eq:loo_htl_bound} seems to have a vanishing behavior whenever the risk of the source $\Rsrc$ is equal to zero.
This comes at the cost of the use of a weaker concentration inequality.
In Theorem~\ref{thm:htl_gen_bound} we manage to obtain the same behavior with high probability.
Finally, we get a better dependency on $\Rsrc$.


%% file: implications_minbound.tex
\subsection{Combining Source Hypotheses in Practice}
So far we have assumed that the problem~\eqref{alg:htl_erm} is supplied with a pre-made combination of source hypotheses, that is, we did not study a particular algorithm for tuning the $\bbeta$ weights.
However, by analyzing our generalization bound~\eqref{thm:htl_gen_bound}, it is easy to come up with algorithms that could be used for this purpose.
In particular, by minizing the bound w.r.t. $\bbeta$, and assuming that the empirical risk $\Rhat_S(\hsrc_{\bbeta})$ converges uniformly to $\Rsrc$,  we have with high probability that
\begin{align*}
  \min_{ \Omega(\bbeta) \leq \rho} \Risk(\htrg_{\bhatw, \bbeta}) &\leq \min_{\Omega(\bbeta) \leq \rho} \left\{
  \Riskh(\htrg_{\bhatw, \bbeta}) + \scO\left(\frac{\kappa \Rhat_S(\hsrc_{\bbeta})}{\sqrt{m} \lambda} + \sqrt{ \frac{\kappa^2 \rho \Rhat_S(\hsrc_{\bbeta})}{m \lambda} }\right) \right\}~.
\end{align*}
Thus, at least theoretically, given a fixed solution $\bhatw$, it is enough to jointly minimize the error of the target hypothesis $\htrg_{\bhatw, \bbeta}$ and the error of the source combination on the target training set.
This is particularly efficient when the square loss is used, since $\bhatw$ can be expresses in terms of an inverse of a covariance matrix that has to be inverted only once~\citep{orabona2009model,tommasi2013learning,kuzborskij2015transfer}.

Many \ac{HTL}-like algorithms can be captured through the above by choosing among different loss functions and regularizers $\Omega$.
The simplest case is just a concatenation of the source hypotheses predictions with the original feature vector.
However, by choosing different regularizers and their parameters, we can treat the source hypotheses in a different way from the original features.
For example, one might enforce sparsity over the source hypotheses, while using the usual L2 regularizer on the target solution $\bhatw$.

%% file: general_results.tex
\section{Technical Results and Proofs}
\label{sec:technical_results}
In this section we present general technical results that are used to prove our theorems.

First, we present the Rademacher complexity generalization bound in Theorem~\ref{the:rad_gen_bound}, which slightly differs from the usual ones.
The difference comes in the assumption that the variance of the loss is uniformly bounded over the hypothesis class.
This will allow us to state a generalization bound that obeys the fast empirical risk convergence rate subject to the small class complexity.
Second, we will also show a generalization bound with the confidence term that vanishes if the complexity of the class is exactly zero.

%
Next, we focus on the Rademacher complexity of the smooth loss function class.
We prove a bound on the empirical Rademacher complexity of a hypothesis class, Lemma~\ref{lem:smooth_loss_class_radh}, that depends on the point-wise bounds on the loss function.
This novel bound might be of independent interest.
Finally, we employ this result to analyze the effect of the source hypotheses on the complexity of the target hypothesis class in Theorem~\ref{the:loss_class_to_R_smooth_loss}.
%
%
%
%
%
%

%% file: proofs.tex
\subsection{Fast Rate Generalization Bound}
\input{vanishing_term.tex}
%
\subsection{Rademacher Complexity of Smooth Loss Class}
\input{complexity.tex}
\subsection{Proofs of Main Results}
\input{proofs_main.tex}

%% file: vanishing_term.tex
The proof of fast-rate and vanishing-confidence-term bounds, Theorem~\ref{the:rad_gen_bound}, stems from the functional generalization of Bennett's inequality which is due to~\citet[Theorem 2.11]{bousquet2002concentration} and that we report here for completeness.
\begin{theorem}[\cite{bousquet2002concentration}]
\label{the:genbennett}
Let $X_1, X_2, \ldots, X_m$ be identically distributed random variables according to $\sD$.
For all $\sD$-measurable, square-integrable $g \in \sG$, with $\E_{X}[g(X)]=0$, and
$\sup_{g \in \sG} \esssup g \leq 1$, we denote
\begin{align}
&Z = \sup_{g \in \sG} \sum_{i=1}^m g(X_i) \label{eq:defZ}.
\end{align}
Let $\sigma$ be a positive real number such that $\sup_{g \in \sG} \Var_{X \sim \sD}[g(X)] \leq \sigma^2$ almost surely.
Then for all $t \geq 0$, we have that
\begin{equation}
\label{eq:genbennett_inequality}
\Prob\left(Z \geq \E[Z] + t\right) \leq \exp\left(-v u\left(\frac{t}{v}\right)\right),
\end{equation}
where
\[\begin{split}
&v = m \sigma^2 + 2\E[Z],\\
&u(y) = (1+y) \log (1+y) - y.
\end{split}\]
\end{theorem}
The following technical lemma will be used to invert the right hand side of ~\eqref{eq:genbennett_inequality}.
\begin{lemma}
\label{lem:log_bound}
Let $a,b>0$ such that $b = (1+a) \log (1+a) - a$. Then $a\leq \frac{3 b}{2 \log(\sqrt{b}+1)}$.
\end{lemma}
\begin{proof}
It is easy to verify that the inverse function $f^{-1}(b)$ of $f(a):=(1+a) \log (1+a) - a$ is
\[
f^{-1}(b) = \exp\left[ W\left(\frac{b-1}{e}\right) +1 \right]-1,
\]
where the function $W:\R_+ \rightarrow \R$ is the Lambert function that satisfies
\[
x=W(x) \exp \left(W(x)\right).
\]
Hence, to obtain an upper bound to $a$, we need an upper bound to the Lambert function.
We use Theorem~2.3 in \citep{hoorfar2008inequalities}, that says that
\[
W(x) \leq \log\frac{x+C}{1+\log(C)}, \quad \forall x> -\frac{1}{e}, \ C>\frac{1}{e}.
\]
Setting $C=\frac{\sqrt{b}+1}{e}$, we obtain
\[
a=f^{-1}(b) 
\leq e \frac{\frac{b-1}{e}+\frac{\sqrt{b}+1}{e}}{1+\log(\frac{\sqrt{b}+1}{e})} -1 
= \frac{b+\sqrt{b}}{\log(\sqrt{b}+1)}-1 \leq \frac{3b}{2\log(\sqrt{b}+1)},
\]
where in the last inequality we used the fact that $x+\sqrt{x} - \log(\sqrt{x}+1) \leq \frac{3}{2} x, \forall x\geq0$, as it can be easily verified comparing the derivatives of both terms.
\qed\end{proof}
The following lemma is a standard tool~\citep[(3.8)-(3.13)]{Mohri2012foundations},~\citep{BartlettM03}.
%
\begin{lemma}[Symmetrization] For any $f \in \sF$, given random variables $S=\{X_i\}_{i=1}^m$, we have
\label{lem:symmetrization}
\begin{align*}
&\E_{S}~ \sup_{f \in \sF} \left\{ \MEbr{X}{f(X)} - \frac{1}{m} \sum_{i=1}^m f(X_i) \right\} \leq 2 \Rad(\sF),\\
&\E_{S}~ \sup_{f \in \sF} \left\{ \frac{1}{m} \sum_{i=1}^m f(X_i) - \MEbr{X}{f(X)} \right\} \leq 2 \Rad(\sF).
\end{align*}
\end{lemma}
Now we are ready to present the proof of Theorem~\ref{the:rad_gen_bound}.
\begin{theorem}
\label{the:rad_gen_bound}
Consider the non-negative loss function $\ell : \sY \times \sY \mapsto \reals_+$, such that $0 \leq \ell(h(\bx), y) \leq M$ for any $h \in \sH$ and any $(\bx, y) \in \sX \times \sY$.
In addition, let the training set $S$ of size $m$ be sampled i.i.d. from the probability distribution over $\sX \times \sY$.
Also for any $r \geq 0$, define the loss class with respect to the hypothesis class $\sH$ as,
\[
\sL := \left\{ (\bx, y) \mapsto \ell(h(\bx), y) : h \in \sH \ \wedge \ \Risk(h) \leq r \right\}.
\]
Then we have for all $h \in \sH$, and any training set $S$ of size $m$, with probability at least $1 - e^{-\tail}, \ \forall \tail \geq 0$
\[
\Risk(h) - \Riskh(h) \leq 2 \Rad(\sL) + \frac{3 M \tail}{m \log\left(1 + \sqrt{\frac{2 M \tail}{v m}}\right)} \leq 2 \Rad(\sL) + 3 \sqrt{\frac{v M \tail}{2m}} + \frac{3 M \tail}{2m},
\]
where $v = 4 \Rad(\sL) + r$.
\end{theorem}
\begin{proof}
To prove the statement, we will consider the uniform deviations of the empirical risk.
Namely, we will show an upper bound on the random variable $\sup_{h \in \sH}\left\{\Risk(h) - \Riskh(h)\right\}$.
For this purpose, we will use the functional generalization of Bennett's inequality given by Theorem~\ref{the:genbennett}.
Consider the random variable
\[
Z := \frac{m}{2 M} \sup_{h \in \sH}\left\{\Risk(h) - \Riskh(h)\right\}.
\]
Using Theorem~\ref{the:genbennett}, we have
\begin{align}
\label{eq:rad_gen_bound_concentration}
&\Prob\left( \frac{m}{2 M} \sup_{h \in \sH}\left\{\Risk(h) - \Riskh(h)\right\} \geq \frac{m}{2 M} \E\left[\sup_{h \in \sH}\left\{\Risk(h) - \Riskh(h)\right\}\right] + t \right) \\
&\quad \leq \exp\left(-v u\left(\frac{t}{v}\right)\right), \nonumber
\end{align}
where,
\begin{align}
v &= m \sigma^2 + \frac{m}{M} \E\left[\sup_{h \in \sH}\left\{\Risk(h) - \Riskh(h)\right\}\right], \label{eq:v_def}\\
\sigma^2 &\geq \sup_{h \in \sH} \Var_{(\bx, y)}\left[\frac{1}{2 M} \left(\ell(h(\bx), y) - \E_{(\bx', y')}[\ell(h(\bx'), y')] \right)\right]. \nonumber
\end{align}
We now need two things: invert the r.h.s. of~\eqref{eq:rad_gen_bound_concentration}, treating it as a function of $t$, and provide an upper-bound on $v$.
For the first part, recall that $u(y) = (1+y) \log (1+y) - y$.
%
To give an upper-bound of $t$, we apply Lemma~\ref{lem:log_bound} with $a=\frac{t}{v}$, and $b=\frac{1}{v}\tail$.
This leads to the inequalities
\begin{equation*}
\label{eq:concentration_log_bound}
\frac{t}{v} \leq \frac{3 \tail}{2 v \log \left(1+\sqrt{\frac{\tail}{v}} \right)} \leq \frac{3\tail}{4v} + \frac{3}{2}\sqrt{\frac{\tail}{v}}.
\end{equation*}
Using this fact, we have with probability at least $1-e^{-\tail}$ with any $\tail \geq 0$
\begin{align}
\frac{m}{2 M} \sup_{h \in \sH}\left\{\Risk(h) - \Riskh(h)\right\} 
&\leq \frac{m}{2 M} \E\left[\sup_{h \in \sH}\left\{\Risk(h) - \Riskh(h)\right\}\right] + \frac{3\tail}{2\log \left(1+\sqrt{\frac{\tail}{v}} \right)} \label{eq:generalization_bounds_with_v_1} \\
&\leq \frac{m}{2 M} \E\left[\sup_{h \in \sH}\left\{\Risk(h) - \Riskh(h)\right\}\right] + \frac{3}{4}\tail + \frac{3}{2} \sqrt{v \tail}. \label{eq:generalization_bounds_with_v_2}
\end{align}

Next we prove the bound on $v$.
We first show that the variance of centered loss function, $\sigma^2$, is uniformly bounded by the Rademacher complexity.
From the definition of variance we have
\begin{align}
&\sup_{h \in \sH}~ \E_{(\bx, y)}\left[ \frac{1}{4 M^2} \left(\ell(h(\bx), y) - \E_{(\bx',y')}[ \ell(h(\bx'), y') ] \right)^2 \right] 
\leq \sup_{h \in \sH}~ \frac{1}{4 M^2} \E_{(\bx, y)}[\ell(h(\bx), y)^2] \nonumber \\
&\qquad \leq \sup_{h \in \sH}~ \frac{1}{2 M} \E_{(\bx, y)}[|\ell(h(\bx), y)|] = \sigma^2 = \sup_{h \in \sH} \frac{1}{2 M} \Risk(h) = \frac{r}{2 M}.
\end{align}
Last inequality is due to the fact that $\ell(h(\bx), y) \leq M$.
%
%
Now we upper-bound the second term of $v$ by applying Lemma~\ref{lem:symmetrization},
\[\begin{split}
&\frac{1}{2 m M} \E_{S}\left[\sup_{h \in \sH} \sum_{i=1}^m \left(\ell(h(\bx_i), y_i) - \E_{(\bx', y')}{[\ell(h(\bx'), y')]} \right) \right]\\
&\qquad = \frac{1}{2 M} \E_S\left[\sup_{h \in \sH} \left\{ \left(\frac{1}{m}  \sum_{i=1}^m \ell(h(\bx_i), y_i) \right) - \E_{(\bx', y')}{[\ell(h(\bx'), y')]} \right\} \right] \leq \frac{1}{M} \Rad(\sL).
\end{split}\]

We conclude the proof by upper-bounding the expectation terms in~\eqref{eq:generalization_bounds_with_v_1} and~\eqref{eq:generalization_bounds_with_v_2} using Lemma~\ref{lem:symmetrization},
and plugging the upper bound on $v$,
\[
v \leq \frac{2 m}{M} \Rad(\sL) + m \sigma^2 \leq \frac{2 m \Rad(\sL)}{M} + \frac{m r}{2 M}.
\]
\qed\end{proof}

%% file: complexity.tex
In this section we study the Rademacher complexity of the hypothesis class populated by functions of the form~\eqref{eq:htl_predictor},
where the parameters $\bw$ and $\bbeta$ are chosen by an algorithm with a strongly convex regularizer.
For this purpose we employ the results of~\citet{kakade2009complexity,kakade2012regularization}, who studied strongly convex regularizers in a more general setting.
Furthermore, we will focus on the use of smooth loss functions as done by~\citet{SrebroST10}.
The proof of the main result of this section, Theorem~\ref{the:loss_class_to_R_smooth_loss}, depends essentially on the following lemma,
that bounds the empirical Rademacher complexity of a $H$-smooth loss class.
%
\begin{lemma}
\label{lem:smooth_loss_class_radh}
Let $\ell : \sY \times \sY \mapsto \reals_+$ be the $H$-smooth loss function.
Then for some function class $\sF$, let the loss class be
\[
\sL = \left\{
(\bx, y) \mapsto \ell(f(\bx), y) : f \in \sF
\right\}.
\]
Then having the sample $S$ of size $m$ and the set
\[
\left\{
\tau_i ~:~ \tau_i \geq \ell(f(\bx_i), y_i),~ \forall (\bx_i, y_i) \in S ~\wedge~ \forall f \in \sF
\right\},
\]
we have that
\[
\Radh(\sL) \leq \E_{\bvarepsilon}\left[
\sup_{f \in \sF}\left\{
\frac{2 \sqrt{3 H}}{m} \sum_{i=1}^m \varepsilon_i \sqrt{\tau_i} f(\bx_i)
\right\}
\right],
\]
where $\varepsilon_i$ is r.v. such that $\mathbb{P}(\varepsilon_i=1) = \mathbb{P}(\varepsilon_i=-1) = \frac{1}{2}$.
\end{lemma}
\begin{proof}
This proof follows a line of reasoning similar to the proof of Talagrand's lemma for Lipschitz functions, see for instance~\citet[p. 79]{Mohri2012foundations}.
We will also use Lemma B.1 by \cite{SrebroST10} (arXiv extended version), stating that for any $H$-smooth non-negative function $\phi : \reals \mapsto \reals_+$ and any $x,z \in \reals$,
\begin{equation}
\label{eq:diff_two_smooth_fn_bound}
|\phi(x) - \phi(z)| \leq \sqrt{6 H (\phi(x) + \phi(z))} |x - z|.
\end{equation}
Fix the sample $S$, then, by definition,
\begin{align*}
\Radh(\sL) &= \frac{1}{m} \E_{\bvarepsilon}\left[
\sup_{f \in \sF}\left\{
\sum_{i=1}^m \varepsilon_i \ell(f(\bx_i), y_i)
\right\}
\right]\\ &=
\E_{\varepsilon_1, \ldots, \varepsilon_{m-1}} \left[
\E_{\varepsilon_m}\left[
\sup_{f \in \sF}\left\{ u_{m-1}(f) + \varepsilon_m \ell(f(\bx_m), y_m) \right\}
\right]
\right],
\end{align*}
where $u_{m-1}(f) = \sum_{i=1}^n \varepsilon_i \ell(f(\bx_i), y_i)$.
By definition of supremum, for any $\delta > 0$, there exist $f_1, f_2 \in \sF$ such that
\begin{align*}
&u_{m-1}(f_1) + \ell(f_1(\bx_m), y_m) \geq (1 - \delta)\left( \sup_{f \in \sF}\left\{ u_{m-1}(f) + \ell(f(\bx), y) \right\} \right)\\
\quad\text{and } &u_{m-1}(f_2) - \ell(f_2(\bx_m), y_m) \geq (1 - \delta)\left( \sup_{f \in \sF}\left\{ u_{m-1}(f) - \ell(f(\bx), y) \right\} \right).
\end{align*}
Thus for any $\delta > 0$, by definition of $\E_{\varepsilon_m}$,
%
\begin{align*}
&(1 - \delta) \E_{\varepsilon_m}\left[ \sup_{f \in \sF}\left\{ u_{m-1}(f) + \varepsilon_m \ell(f(\bx_m), y_m) \right\} \right] \\
&\qquad = \frac{1 - \delta}{2} \left( \sup_{f \in \sF}\left\{ u_{m-1}(f) + \ell(f(\bx_m), y_m) \right\} + \sup_{f \in \sF}\left\{ u_{m-1}(f) - \ell(f(\bx_m), y_m) \right\}\right) \\
&\qquad \leq \frac{1}{2} \bigg( u_{m-1}(f_1) + \ell(f_1(\bx_m), y_m) + u_{m-1}(f_2) - \ell(f_2(\bx_m), y_m) \bigg) \\
&\qquad \leq \frac{1}{2} \bigg( u_{m-1}(f_1) + u_{m-1}(f_2) \\
&\qquad\qquad+ s_m \sqrt{6 H \ell(f_1(\bx_m), y_m)+ \ell(f_2(\bx_m), y_m) } (f_1(\bx_m) - f_2(\bx_m)) \bigg) \\
&\qquad \leq \frac{1}{2} \bigg( u_{m-1}(f_1) + u_{m-1}(f_2) + s_m \sqrt{12 H \tau_m } (f_1(\bx_m) - f_2(\bx_m)) \bigg) \\
&\qquad \leq \frac{1}{2} \sup_{f \in \sF} \bigg\{ u_{m-1}(f) + s_m \sqrt{12 H \tau_m } f(\bx_m) \bigg\}\\
&\qquad+ \frac{1}{2} \sup_{f \in \sF} \bigg\{ u_{m-1}(f) - s_m \sqrt{12 H \tau_m } f(\bx_m) \bigg\} \\
&\qquad= \E_{\varepsilon_m}\left[\sup_{f \in \sF} \left\{ u_{m-1}(f) + \varepsilon_m \sqrt{12 H \tau_m } f(\bx_m)\right\}\right].
\end{align*}
To obtain the second inequality, we applied~\eqref{eq:diff_two_smooth_fn_bound}, where $s_m = \sgn(f_1(\bx_m) - f_2(\bx_m))$.
Since the inequality holds for all $\delta > 0$, we have
\[
\E_{\varepsilon_m}\left[ \sup_{f \in \sF}\left\{ u_{m-1}(f) + \varepsilon_m \ell(f(\bx_m), y_m) \right\} \right] \leq \E_{\varepsilon_m}\left[ \sup_{f \in \sF} \left\{ u_{m-1}(f) + \varepsilon_m \sqrt{12 H \tau_m } f(\bx_m) \right\}\right].
\]
Proceeding in the same way for all the other $\varepsilon_i$, with $i \neq m$, proves the lemma.
\qed\end{proof}
To prove Theorem~\ref{the:loss_class_to_R_smooth_loss} we will also use the following lemma in~\citet[Corollary 4]{kakade2012regularization}.
\begin{lemma}[\cite{kakade2012regularization}]
\label{lem:ip_strong_conv}
If $\regul$ is $\sigma$ strongly convex w.r.t. $\|\cdot\|$ and $\regul^\star(\bzero) = 0$, then,
denoting the partial sum $\sum_{j \leq i} \bv_j$ by $\bv_{1:i}$,
we have for any sequence $\bv_1, \ldots, \bv_m$ and for any $\bu$,
\[
\sum_{i=1}^m \ip{\bv_i, \bu} - \regul(\bu) \leq \regul^\star(\bv_{1:m}) \leq \sum_{i=1}^m \ip{\nabla \regul^\star(\bv_{1:i-1}), \bv_i} + \frac{1}{2 \sigma} \sum_{i=1}^m \|\bv_i\|_\star^2~.
\]
\end{lemma}
Now we are ready to give the proofs of the Rademacher complexity results.
%
%
%
%
\begin{theorem}
\label{the:loss_class_to_R_smooth_loss}
Let $\regul$ be a non-negative $\sigma$-strongly convex function w.r.t. a norm $\|\cdot\|$, and let $\bzero$ be its minimizer.
Let risk and empirical risk be defined w.r.t. an $H$-smooth loss function $\ell : \sY \times \sY \mapsto \reals_+$.
Finally, given the set of functions $\{f_i : \sX \mapsto \sY\}_{i=1}^n$ with $\mathbf{f}(\bx) := [f_1(\bx), \ldots, f_n(\bx)]\tp$, a combination $f_{\bbeta}(\bx) = \ip{\bbeta, \mathbf{f}(\bx)}$,
a scalar $\alpha > 0$, and any sample $S$ drawn i.i.d. from distribution over $\sX \times \sY$, define classes
\begin{equation*}
\sW = \left\{ \bw ~:~ \regul(\bw) \leq \alpha \Riskh(f_{\bbeta}) \right\}, \quad \sV = \left\{ {\bbeta} ~:~ \regul({\bbeta}) \leq \rho \right\},
\end{equation*}
and the loss class
\begin{equation*}
\sL = \left\{ (\bx, y) \mapsto \ell(\ip{\bw, \bx} + f_{\bbeta}(\bx), y) ~:~ \bw \in \sW \ \wedge \ {\bbeta} \in \sV \right\}.
\end{equation*}
Then for the loss class $\sL$, setting constants $\sup_{\bx \in \sX} \|\bx\|_\star \leq B$ and $\sup_{\bx \in \sX} \|\mathbf{f}(\bx)\|_\star \leq C$,
we have that
\begin{equation*}
  \Rad(\sL) \leq 4 \sqrt{3 H} (B + C) \left( 1 + \sqrt{\frac{2 H B^2 \alpha}{\sigma}} \right) \frac{\Risk(f_{\bbeta}) \sqrt{\alpha} + \sqrt{\Risk(f_{\bbeta}) \rho}}{\sqrt{m \sigma}}.
\end{equation*}
\end{theorem}
\begin{proof}
The core of the proof consists in an application of Lemma~\ref{lem:smooth_loss_class_radh}.
In particular, Lemma~\ref{lem:smooth_loss_class_radh} allows us to introduce additional information about the loss class by providing bounds on the loss at each example.
We will bound the loss at each example using the definition of smoothness, extracting the empirical risk of hypothesis $\Riskh(f_{\bbeta})$.
The last step is to give an upper-bound on the empirical Rademacher complexity of a class regularized by a strongly convex function.
We follow the proof of~\citet[Theorem 7]{kakade2012regularization} to accomplish this task.
First define the classes
\[
\sH_\sW := \left\{ \bx \mapsto \ip{\bw, \bx} \ : \ \bw \in \sW \right\}, \quad \sH_\sV := \left\{ f_{\bbeta} \ : \ {\bbeta} \in \sV \right\},
\]
%
%
%
%
and also define altered samples $S' := \{ \sqrt{\tau_i} \bx_i \}_{i=1}^m$ and $S'' := \{ \sqrt{\tau_i} \mathbf{f}(\bx_i) \}_{i=1}^m$, where $\tau_i$ is a quantity independent from $\sW$ and $\sV$.
Then by applying Lemma~\ref{lem:smooth_loss_class_radh}, we have that,
\begin{align*}
&\Radh(\sL) \leq \E_{\bvarepsilon}\left[ \sup_{\substack{\bw \in \sW\\ \bbeta \in \sV}} \left\{ \frac{2 \sqrt{3 H}}{m} \sum_{i=1}^m \varepsilon_i \sqrt{\tau_i} \ip{\bw, \bx_i}
+ \frac{2 \sqrt{3 H}}{m} \sum_{i=1}^m \varepsilon_i \sqrt{\tau_i} \ip{\bbeta, \mathbf{f}(\bx_i)} \right\} \right]\\
&\leq \E_{\bvarepsilon}\left[ \sup_{\bw \in \sW} \left\{ \frac{2 \sqrt{3 H}}{m} \sum_{i=1}^m \varepsilon_i \sqrt{\tau_i} \ip{\bw, \bx_i} \right\} \right] +
\E_{\bvarepsilon}\left[ \sup_{\bbeta \in \sV} \left\{ \frac{2 \sqrt{3 H}}{m} \sum_{i=1}^m \varepsilon_i \sqrt{\tau_i} \ip{\bbeta, \mathbf{f}(\bx_i)} \right\} \right]\\
&= \RadhNoS{S'}(\sH_{\sW}) + \RadhNoS{S''}(\sH_{\sV}).
\end{align*}
%
%
%
%
%
%
Having this, we will follow the proof of~\citet[Theorem 7]{kakade2012regularization} to bound the empirical Rademacher complexities $\RadhNoS{S'}(\sH_{\sW})$ and $\RadhNoS{S''}(\sH_{\sV})$ with quantities of interest.
Let $t > 0$ and apply Lemma~\ref{lem:ip_strong_conv} with $\bu = \bw$ and $\bv_i = t \varepsilon_i \sqrt{\tau_i} \bx_i$ to get
\begin{align*}
&\sup_{\bw \in \sW}\left\{ \sum_{i=1}^m \ip{\bw, t \varepsilon_i \sqrt{\tau_i} \bx_i} \right\}\\
&\qquad\leq \frac{t^2}{2 \sigma} \sum_{i=1}^m \|\varepsilon_i \sqrt{\tau_i} \bx_i\|_\star^2 + \sup_{\bw \in \sW} \regul(\bw) + \sum_{i=1}^m \ip{\nabla \regul^\star(\bv_{1:i-1}), \varepsilon_i \sqrt{\tau_i} \bx_i}\\
&\qquad\leq \frac{t^2 B^2}{2 \sigma} \sum_{i=1}^m |\tau_i| + \alpha \Riskh(f) + \sum_{i=1}^m \ip{\nabla \regul^\star(\bv_{1:i-1}), \varepsilon_i \sqrt{\tau_i} \bx_i}.
\end{align*}
Now take expectation w.r.t. all the $\varepsilon_i$ on both sides.
The left hand side is $m t \RadhNoS{S'}(\sH_{\sW})$ and the last term on the right hand side becomes zero since $\E[\varepsilon_i]=0$.
Denoting $r = \frac{1}{m} \sum_{i=1}^m |\tau_i|$ and multiplying through by $\frac{1}{m t}$, we get
\begin{equation*}
\RadhNoS{S'}(\sH_{\sW}) \leq \frac{B^2 r t}{2 \sigma} + \frac{\alpha}{m t} \Riskh(f_{\bbeta}).
\end{equation*}
Proving analogously for $\RadhNoS{S''}(\sH_{\sV})$, we get that
\begin{equation*}
  \Radh(\sL) \leq 2 \sqrt{3 H}\left( \frac{(B^2 + C^2) r t}{\sigma} + \frac{\alpha \Riskh(f_{\bbeta}) + \rho}{m t} \right).
\end{equation*}
Optimizing over $t$ gives us
\begin{equation*}
  \Radh(\sL) \leq 4 \sqrt{3 H} (B + C) \sqrt{\frac{ r (\alpha \Riskh(f_{\bbeta}) + \rho) }{m \sigma}}.
\end{equation*}
Now focus on the upper bound of $r$.
First we obtain bounds on each $\tau_i$. 
We start with the bound on the loss function, exploiting smoothness.
Let $\ell(\ip{\bw, \bx} + f_{\bbeta}(\bx), y) = \phi(\ip{\bw, \bx} + f_{\bbeta}(\bx))$, where $\phi : \reals \mapsto \reals$ is an $H$-smooth function.
From the definition of smoothness~\citep[(12.5)]{shalev2014understanding}, we have that for all $\bw$ and $\bv$
\begin{align}
&\phi(\ip{\bw, \bx} + f_{\bbeta}(\bx))\nonumber\\
&\quad\leq \phi(\ip{\bv, \bx} + f_{\bbeta}(\bx)) + \phi'(\ip{\bv, \bx} + f_{\bbeta}(\bx)) \ip{\bw - \bv, \bx} + \frac{H}{2} \ip{\bw - \bv, \bx}^2 \nonumber\\
&\quad\leq \phi(\ip{\bv, \bx} + f_{\bbeta}(\bx)) + 2 \sqrt{H \phi(\ip{\bv, \bx} + f_{\bbeta}(\bx))} \|\bw - \bv\| \|\bx\|_\star + \frac{H}{2} \|\bw - \bv\|^2 \|\bx\|_\star^2. \label{eq:smooth_loss_gen_CS}
\end{align}
To obtain the last inequality we used the generalized Cauchy-Schwarz inequality and 
the fact that for an $H$-smooth non-negative function $\phi$, we have that
$|\phi'(t)| \leq \sqrt{4 H \phi(t)}$,~\citep[Lemma 2.1]{SrebroST10}.
Now recall a property of a $\sigma$-strongly-convex function $F$, that holds for its minimizer $\bv$ and any $\bw$~\citep[Lemma 13.5]{shalev2014understanding},
\[
\|\bw - \bv\|^2 \leq \frac{2}{\sigma}(F(\bw) - F(\bv)).
\]
Since inequality~\eqref{eq:smooth_loss_gen_CS} holds for any $\bv$, set $\bv = \bzero$, which is also the minimizer of $\regul(\cdot)$,
apply aforementioned property to get
\begin{align}
&\phi(\ip{\bw, \bx} + f_{\bbeta}(\bx)) \leq \phi(f_{\bbeta}(\bx)) + 2 \sqrt{\frac{2 H}{\sigma} \phi(f_{\bbeta}(\bx)) \regul(\bw)} \|\bx\|_\star + \frac{H}{\sigma} \regul(\bw) \|\bx\|_\star^2~ \nonumber\\
\Rightarrow~ &\ell(\ip{\bw, \bx_i} + f_{\bbeta}(\bx_i), y_i) \leq \tau_i\\
&= \ell(f_{\bbeta}(\bx_i), y_i) + \sqrt{\frac{8 H B^2 \alpha}{\sigma} \Riskh(f_{\bbeta}) \ell(f_{\bbeta}(\bx_i), y_i)} + \frac{H B^2 \alpha}{\sigma} \Riskh(f_{\bbeta}). \label{eq:multisrc_smooth_loss_bound}
\end{align}
The last inequality comes from the definition of the class $\sH$.
%
%
%
%
%
Now we consider the average and, by Jensen's inequality,
\begin{align*}
r &= \frac{1}{m} \sum_{i=1}^m |\tau_i| = \Riskh(f_{\bbeta}) + \frac{1}{m} \sum_{i=1}^m \sqrt{\frac{8 H B^2 \alpha}{\sigma} \Riskh(f_{\bbeta}) \ell(f_{\bbeta}(\bx_i), y_i)} + \frac{H B^2 \alpha}{\sigma} \Riskh(f_{\bbeta})\\
&\leq \Riskh(f_{\bbeta}) + \sqrt{\frac{8 H B^2 \alpha}{\sigma}} \Riskh(f_{\bbeta}) + \frac{H B^2 \alpha}{\sigma} \Riskh(f_{\bbeta})
\leq \left( 1 + \sqrt{\frac{2 H B^2 \alpha}{\sigma}} \right)^2 \Riskh(f_{\bbeta}).
\end{align*}
This gives us
\begin{align*}
  \Radh(\sL) &\leq 4 \sqrt{3 H} (B + C) \left( 1 + \sqrt{\frac{2 H B^2 \alpha}{\sigma}} \right) \sqrt{\frac{\Riskh(f_{\bbeta}) (\alpha \Riskh(f_{\bbeta}) + \rho)}{m \sigma}}~\\
  &\leq 4 \sqrt{3 H} (B + C) \left( 1 + \sqrt{\frac{2 H B^2 \alpha}{\sigma}} \right) \frac{\Riskh(f_{\bbeta}) \sqrt{\alpha} + \sqrt{\Riskh(f_{\bbeta}) \rho}}{\sqrt{m \sigma}}.
\end{align*}
Taking expectation w.r.t. the sample on both sides and applying Jensen's inequality gives the statement.
\qed\end{proof}

%% file: proofs_main.tex
\begin{proof_custom}{Proof of Theorem~\ref{thm:htl_gen_bound}.}
  To show the statement we will apply Theorem~\ref{the:rad_gen_bound}.
  In particular, we will consider any choice of $\bw$ and $\bbeta$ within the set induced by a strongly-convex function $\Omega$.
  To apply Theorem~\ref{the:rad_gen_bound}, we need to upper bound the Rademacher complexity of the loss class $\sL$
  and also the quantity $r = \sup_{f \in \sL} \E_{(\bx,y)} [f(\bx, y)]$.

  We obtain the bound on Rademacher complexity by applying Theorem~\ref{the:loss_class_to_R_smooth_loss}.
  First define the loss class
  $
  \sL := \left\{(\bx, y) \mapsto \ell(h, y) \ : h \in \sH \right\},
  $
  and hypothesis class
  \begin{align*}
  \sH := \Big\{ &\bx \mapsto \ip{\bw, \bx} + \hsrc_{\bbeta}(\bx) \ : \ \\
  &\quad\Omega(\bw) \leq \frac{1}{\lambda} \Riskh(\hsrc_{\bbeta}) \ \wedge \ \Omega(\bbeta) \leq \rho \ \wedge \ \Riskh(h_{\bw, \bbeta}) \leq \Riskh(\hsrc_{\bbeta})  \Big\}.
  \end{align*}
  To motivate the choice for the constraints observe that for
  \[
  \bhatw = \argmin_{\bw}\left\{ \Riskh(h_{\bw, \bbeta}) + \lambda \regul(\bw) \right\},
  \]
  we have $\regul(\bhatw) \leq \lambda^{-1} \Riskh(h_{\bzero, \bbeta}) = \lambda^{-1} \Riskh(\hsrc_{\bbeta})$, and $\Riskh(h_{\bhatw, \bbeta}) \leq \Riskh(\hsrc_{\bbeta})$.
  That said, by applying Theorem~\ref{the:loss_class_to_R_smooth_loss} with $\alpha = \frac{1}{\lambda}$ and $f_{\bbeta}=\hsrc_{\bbeta}$ and assuming that $\lambda \leq \kappa$, we obtain
  \begin{equation*}
    \Rad(\sL) \leq \scO\left( \frac{\Rsrc \kappa}{\sqrt{m} \lambda} + \sqrt{\frac{\Rsrc \rho \kappa^2}{m \lambda}} \right).
  \end{equation*}

  Next we obtain the bound on $r$
  \begin{align*}
    &r = \sup_{h \in \sH}\E_{(\bx,y)} [\ell(h(\bx), y)] = \sup_{h \in \sH} \E_{S} \left[ \Riskh(h) \right]
    \leq \E_{S} \left[ \sup_{h \in \sH} \Riskh(h) \right] \leq  \E_{S} [\Riskh(\hsrc_{\bbeta})] =  \Rsrc.
  \end{align*}
  The last two inequalities come from Jensen's inequality and the definition of the class $\sH$.
  Plugging the bounds on the Rademacher complexity and $r$ into the statement of Theorem~\ref{the:rad_gen_bound},
  and applying the inequality $\sqrt{a+b} \leq \sqrt{a}+ \frac{b}{2\sqrt{a}}$ to the $\sqrt{v}$ term, gives the statement.
\end{proof_custom}

\begin{proof_custom}{Proof of Theorem~\ref{thm:excess_risk}.}
  For any choice of $\bbeta$ with $\regul(\bbeta) \leq \rho$, denote the best in the class by
  \[
  \bw^\star = \argmin_{\bw ~:~ \regul(\bw) \leq \tau} \Risk(h_{\bw, \bbeta}).
  \]
  By the definition of $\bhatw$, we have
  \begin{equation}
    \label{eq:excess_objective_of_best}
    \Riskh(h_{\bhatw, \bbeta}) + \lambda \regul(\bhatw) \leq \Riskh(h_{\bw^\star, \bbeta}) + \lambda \regul(\bw^\star).
  \end{equation}
  Now denote
  \begin{align*}
    Z = \kappa \sqrt{\frac{\Rsrc}{m}} (\sqrt{\Rsrc} + \sqrt{\rho}).
  \end{align*}
  Then, by following the proof of Theorem~\ref{thm:htl_gen_bound} until the application of inequality $\sqrt{a+b} \leq \sqrt{a} + \frac{b}{2 \sqrt{a}}$, ignoring constants, using the assumption~\eqref{eq:excess_objective_of_best}, and assuming that $\lambda \leq \kappa \leq 1$ we have that
  \begin{align}
    \Risk(\htrg_{\bhatw, \bbeta}) &\leq \Riskh(\htrg_{\bw^\star, \bbeta}) + \lambda \tau + \frac{Z}{\lambda} + \sqrt{\frac{M \tail}{m}} \sqrt{\Rsrc + \frac{Z}{\lambda}} + \frac{M \tail}{m} \nonumber \\
    &\leq \Riskh(\htrg_{\bw^\star, \bbeta}) + \lambda \tau + \frac{Z}{\lambda} + \sqrt{\frac{\Rsrc M \tail}{m}} + \frac{\sqrt{Z M \tail}}{\sqrt{m} \lambda} + \frac{M \tail}{m}. \label{eq:excess_risk_proof_1}
  \end{align}
  Optimizing the l.h.s. over $\lambda$ gives
  \begin{equation*}
    \lambda^\star = \sqrt{\frac{Z}{\tau} + \frac{1}{\tau}\sqrt{\frac{Z M \tail}{m}}}.
  \end{equation*}
  We plug it back into~\eqref{eq:excess_risk_proof_1} to obtain that
  \begin{align}
    \Risk(\htrg_{\bhatw, \bbeta}) &\leq \Riskh(\htrg_{\bw^\star, \bbeta}) + \sqrt{\tau} \sqrt{Z + \sqrt{\frac{Z M \tail}{m}}} + \sqrt{\frac{\Rsrc M \tail}{m}} + \frac{M \tail}{m} \nonumber \\
    &\leq \Riskh(\htrg_{\bw^\star, \bbeta}) + \frac{\sqrt{\Rsrc} + \sqrt[4]{\Rsrc \rho}}{\sqrt[4]{m}} \sqrt{\kappa \tau} + \frac{\sqrt[4]{\Rsrc} + \sqrt[8]{\Rsrc \rho}}{\sqrt[4]{m^{1.5}}} \sqrt[4]{\kappa \tau^2 M \tail} \nonumber\\
    &\quad+ \sqrt{\frac{\Rsrc M \tail}{m}} + \frac{M \tail}{m}. \label{eq:excess_risk_proof_2}
  \end{align}
  All that is left is to concentrate $\Riskh(\htrg_{\bw^\star, \bbeta})$ around its mean.
  Denoting the variance by
  \begin{align*}
    V = \E\left[\sum_{i=1}^m (\ell(h_{\bw^\star, \bbeta}(\bx_i), y_i) - \Risk(h_{\bw^\star, \bbeta}))^2\right],
  \end{align*}
    we apply Bernstein's inequality
    \begin{equation*}
      \Prob\left( \sum_{i=1}^m (\ell(h_{\bw^\star, \bbeta}(\bx_i), y_i) - \Risk(h_{\bw^\star, \bbeta})) > t \right) \leq \exp\left( - \frac{t^2 / 2}{V + M t/3} \right).
    \end{equation*}
    Setting
    \[
    e^{-\eta} = \exp\left( - \frac{t^2 / 2}{V + M t/3} \right),
    \]
    we have that with probability at least $1-e^{-\eta}, \ \forall \eta \geq 0$
    \begin{align*}
      \Riskh(h_{\bw^\star, \bbeta}) &\leq \Risk(h_{\bw^\star, \bbeta}) + \sqrt{\frac{2 \tail \E\left[(\ell(h_{\bw^\star, \bbeta}(\bx_i), y_i) - \Risk(h_{\bw^\star, \bbeta}))^2\right]}{m}} + \frac{2 M \eta}{3 m} \\
      &\leq \Risk(h_{\bw^\star, \bbeta}) + 2 \sqrt{\frac{\Risk(h_{\bw^\star, \bbeta}) M \tail}{m}} + \frac{2 M \eta}{3 m} \\
      &\leq \Risk(h_{\bw^\star, \bbeta}) + 2 \sqrt{\frac{\Rsrc M \tail}{m}} + \frac{2 M \eta}{3 m}.
    \end{align*}
    The last inequality comes from the observation that $\Risk(h_{\bw^\star, \bbeta}) \leq \Risk(h_{\bzero}) = \Rsrc$.
    Plugging this result into~\eqref{eq:excess_risk_proof_2} completes the proof.
\end{proof_custom}

%% file: conclusions.tex
\section{Conclusions}
\label{sec:conc}
In this paper we have formally captured and theoretically analyzed a general family of learning algorithms
transferring information from multiple supplied source hypotheses.
In particular, our formulation stems from the regularized \acl{ERM} principle with the choice of any non-negative smooth loss function and any strongly convex regularizer.
Theoretically we have analyzed the generalization ability and excess risk of this family of \ac{HTL} algorithms.
Our analysis showed that a good source hypothesis combination facilitates faster generalization, specifically in $\scO(1/m)$ instead of the usual $\scO(1/\sqrt{m})$.
Furthermore, given a perfect source hypothesis combination, our analysis is consistent with the intuition that learning is not required.
As a byproduct of our investigation, we came up with new results in Rademacher complexity analysis of the smooth loss classes,
that could be of independent interest.

Our conclusions suggest the key importance of a source hypothesis selection procedure.
Indeed, when an algorithm is provided with enormous pool of source hypotheses, how to select relevant ones on the basis of only few labeled examples?
This might sound similar to the feature selection problem under the condition that $n \gg m$, however, earlier empirical studies by~\cite{tommasi2013learning} with hundreds of sources did not find much corroboration for this hypothesis when applying $L1$ regularization.
Thus, it remains unclear if having few good sources from hundreds is a reasonable assumption.
%
%
%
%
%
%
%

%% file: appendix.tex
\appendix
\section{Additional Proofs}
\begin{theorem}
\label{thm:srebro_gen_bound}
Let $\htrg_{\bhatw, \bbeta}$ be generated by \algname, given the $m$-sized training set $S$ sampled i.i.d. from the target domain, source hypotheses $\{\hsrc_i\}_{i=1}^n$, any source weights $\bbeta$ obeying $\regul(\bbeta) \leq \rho$, and $\lambda \in \reals_+$.
Assume that $\ell(\htrg_{\bhatw, \bbeta}(\bx), y) \leq M$ for any $(\bx, y)$ and any training set.
Then, denoting $\kappa = \frac{H}{\sigma}$ and assuming that $\lambda \leq 1$,
we have with probability at least $1 - e^{-\tail}, \ \forall \tail \geq 0$
\begin{align*}
\Risk(\htrg_{\bhatw, \bbeta}) \leq \Riskh(\htrg_{\bhatw, \bbeta}) + \mathcal{\tilde{O}}\Bigg( &\left(\sqrt{\frac{\Rsrc}{m}} + \sqrt[4]{\frac{M^2 \rho}{m^3 \sigma}} + \sqrt[8]{\frac{M^4 \rho}{m^7 \lambda^2 \sigma^3}} \right)
 \left(\sqrt{\frac{M \kappa}{\lambda}} + \sqrt{\kappa \rho} + \sqrt{M \tail}\right) \Bigg).
\end{align*}
\end{theorem}
\begin{proof}
To prove the statement we will use Theorem 1 of~\cite{SrebroST10}.
In particular, we need to obtain bounds on the empirical risk and also to bound the worst case Rademacher complexity of the class
\begin{equation*}
\sH = \left\{ \bx \mapsto \ip{\bw, \bx} + \hsrc_{\bbeta}(\bx) \ : \ \regul(\bw) \leq \frac{\Riskh(\hsrc_{\bbeta})}{\lambda} \ \wedge \ \regul(\bbeta) \leq \tau \right\}.
\end{equation*}
The corresponding loss class is
\begin{equation*}
\sL = \left\{ (\bx, y) \mapsto \ell(h(\bx), y) \ : \ h \in \sH \ \wedge \ \Risk(h) \leq \Rsrc  \right\}.
\end{equation*}
A constraint on $\regul(\bbeta)$ in $\sH$ comes from the statement of the theorem, while
a constraint on $\regul(\bhatw)$ comes from an observation that for
\begin{equation*}
\bhatw = \argmin_{\bw} \left\{ \Riskh(h_{\bw, \bbeta}) + \lambda \regul(\bw) \right\},
\end{equation*}
so we have  $\regul(\bhatw) \leq \frac{\Riskh(h_{\bzero, \bbeta})}{\lambda}$.
The same argument immediately gives us a bound on the empirical risk, that is,
$\Riskh(h_{\bhatw, \bbeta}) \leq \Riskh(h_{\bzero, \bbeta}) = \Riskh(\hsrc_{\bbeta})$.
Taking expectation on both sides gives the constraint of $\sL$.

By applying Theorem~1 of~\cite{kakade2009complexity} and subadditive property of Rademacher complexities~\citep{BartlettM03}, we have that
\begin{equation}
\label{eq:srebro_proof_1}
\Radh(\sH) \leq \sqrt{\frac{2 \Riskh(\hsrc_{\bbeta})}{m \lambda \sigma}} + \sqrt{\frac{2 \rho}{m \sigma}} \leq \sqrt{\frac{2 M}{m \lambda \sigma}} + \sqrt{\frac{2 \rho}{m \sigma}}.
\end{equation}
Note that the upper bound is the bound on the worst-case Rademacher complexity since no term depends on the sample.

All that is left to do is to show the bound on the empirical risk in terms of $\Rsrc$.
However, we cannot use Theorem 1 of~\cite{SrebroST10} since it is not symmetric.
Instead we will use a similar localized bound of~\citet[Corollary 3.5]{bartlett2005local}.
In order to apply it, we have to obtain an upper bound on the Rademacher complexity of the loss class $\sL$ that is a sub-root function~\cite[Definition 4.1]{bousquet2002concentration}.
By using the fact that loss function is bounded, we apply Talagrand's lemma~\citep{Mohri2012foundations}, have $\Radh(\sL) \leq M \Radh(\sH)$, upper-bound with the first inequality of~\eqref{eq:srebro_proof_1} and applying Jensen's inequality w.r.t. $\E[\cdot]$ have
\begin{equation*}
\Rad(\sL) \leq M \sqrt{\frac{2 \Rsrc}{m \lambda \sigma}} + M \sqrt{\frac{2 \rho}{m \sigma}}.
\end{equation*}
Since upper bound is a sub-root function of $\Rsrc$, we obtain it's fixed point $r^\star$ as required by Corollary 3.5 and conclude that
\begin{equation*}
r^\star \leq \sqrt{\frac{2 M^2 \rho}{m \sigma}} + \frac{2 M^2}{m \lambda \sigma} + 2 M \sqrt[4]{\frac{8 \rho}{m^3 \lambda^2 \sigma^3}}.
\end{equation*}
Now we apply Corollary 3.5 and for any $K > 0$ we have with probability at least $1 - e^{-\tail}, \ \forall \tail \geq 0$ the following holds
\begin{equation*}
\Riskh(\htrg_{\bhatw, \bbeta}) \leq K\left( \Rsrc + \sqrt{\frac{M^2 \rho}{m \sigma}} + \frac{M^2}{m \lambda \sigma} + M \sqrt[4]{\frac{\rho}{m^3 \lambda^2 \sigma^3}} + \frac{1 + \tail}{m} \right).
\end{equation*}
All that is left to do is to apply Theorem 1 of~\cite{SrebroST10} to have
\begin{align*}
\Risk(\htrg_{\bhatw, \bbeta}) \leq \Riskh(\htrg_{\bhatw, \bbeta}) + \mathcal{\tilde{O}}\Bigg( &\left(\sqrt{\frac{\Rsrc}{m}} + \sqrt[4]{\frac{M^2 \rho}{m^3 \sigma}} + \frac{M}{m \sqrt{\lambda \sigma}} + \sqrt[8]{\frac{M^2 \rho }{m^7 \lambda^2 \sigma^3}} + \frac{\sqrt{1 + \eta}}{m}\right) \times \\
&\times \left(\sqrt{\frac{M \kappa}{\lambda}} + \sqrt{\kappa \rho} + \sqrt{M \tail}\right)
+ \frac{M \kappa}{m \lambda} + \frac{\kappa \rho}{m} \Bigg).
\end{align*}
Using the assumption on $\lambda$, we get the stated result.
%
%
\qed\end{proof}
\subsection{Guarantees using Localized Rademacher Complexity Bounds}
\input{guarantees_localized.tex}


%% file: guarantees_localized.tex
The following theorem is due to~\citet[Theorem 6.1]{bousquet2002concentration}.
In particular, we state the inequality appearing prior to the last in the proof, as it better serves our purpose.
\begin{theorem}[\cite{bousquet2002concentration}]
  \label{eq:bousquet_local}
  Let $\sF$ be a class of non-negative functions such that $\|f\|_\infty \leq M$ almost surely.
  Let $\phi_m$ be a function defined on $[0, \infty)$ that is non-negative, non-decreasing, not identically zero, and
    such that $\phi_m(r)/\sqrt{r}$ is non-increasing.
  Moreover let $\phi_m$ be such that for all $r > 0$
  \[
  \Radh(\sF) \leq \phi_m(r).
  \]
  Define $r^\star_m$ as the largest solution of the equation $\phi_m(r) = r$.Then, for all $\tail > 0$, with probability at least $1 - e^{-\tail}$ for all $f \in \sF$ and any $\{X_i\}_{i=1}^m$ drawn i.i.d.
  \begin{align*}
    \E_X[f(X)] &\leq \frac{1}{m} \sum_{i=1}^m f(X_i) + 45 r^\star_m + \sqrt{8 r^\star_m \E_X[f(X)]} + \sqrt{\frac{4 M (\tail + 6 \log \log m) \E_X[f(X)]}{m}}\\
    &\quad+ \frac{20 M (\tail + 6 \log \log m)}{m}~.
  \end{align*}
\end{theorem}
The following \ac{HTL} generalization bound is shown using Theorem~\ref{eq:bousquet_local}.
\begin{theorem}
\label{thm:htl_via_localized_bounds}
Let $\htrg_{\bhatw, \bbeta}$ be generated by \algname, given the $m$-sized training set $S$ sampled i.i.d. from the target domain, source hypotheses $\{\hsrc_i\}_{i=1}^n$, any source weights $\bbeta$ obeying $\regul(\bbeta) \leq \rho$, and $\lambda \in \reals_+$.
Assume that $\ell$ is a $L$-Lipschitz loss function and $\ell(\htrg_{\bhatw, \bbeta}(\bx), y) \leq M$ for any $(\bx, y)$ and any training set.
Then we have with probability at least $1 - e^{-\tail}, \ \forall \tail \geq 0$
  \begin{align*}
    \Risk(h_{\bhatw, \bbeta}) \leq \Riskh(h_{\bhatw, \bbeta}) + \tilde{\scO}\Bigg(&
\frac{L^2 + L}{m \lambda \sigma} + L \sqrt{\frac{\rho}{m \sigma}} + \sqrt{\frac{\Rsrc (L^2 + L)}{m \lambda \sigma}}\\
&\quad+ \sqrt{\Rsrc} \sqrt[4]{\frac{L^2 \rho}{m \sigma}} + \sqrt{\frac{\Rsrc M \tail}{m}} + \frac{M \tail}{m}
\Bigg).
  \end{align*}
\end{theorem}
\begin{proof}
  The core of the proof is an application of Theorem~\ref{eq:bousquet_local}.
  In particular, we have to obtain the fixed point $r^\star_m$ and upper bound $\Risk(h)$ with the risk of the source hypothesis $\Rsrc$.

  Considering the $L$-Lipschitz loss class of Theorem~\ref{eq:bousquet_local} to be $\sL := \{(\bx, y) \mapsto \ell(h(\bx), y) ~:~ h \in \sH\}$,
  we have the relationship $\Radh(\sL) \leq L \Radh(\sH)$ via Talagrand's lemma~\citep[Lemma 4.2]{Mohri2012foundations}.
  Furthermore, let the hypothesis class be
  \[
  \sH = \left\{ \bx \mapsto \ip{\bw, \bx} + \hsrc_{\bbeta}(\bx) \ : \ \Omega(\bw) \leq \frac{1}{\lambda} \Riskh(\hsrc_{\bbeta}) \ \wedge \ \Omega(\bbeta) \leq \rho \ \wedge \ \Riskh(h_{\bw, \bbeta}) \leq \Riskh(\hsrc_{\bbeta})  \right\}.
  \]
  The motivation for the choice of constraints comes from the same argument as in the proof of Theorem~\ref{thm:htl_gen_bound}.
  That said, we obtain the upper bound
\[
\Radh(\sL) \leq L \sqrt{\frac{2 \Rsrc}{m \lambda \sigma}} + L \sqrt{\frac{2 \rho}{m \sigma}}~.
\]
Both terms come by applying Theorem 7 by~\citet{kakade2012regularization}.
In the first term we set $f_{\text{max}} = \Rsrc$ and in the second $f_{\text{max}} = \rho$.
Now define function $\phi_m(r) = L \sqrt{\frac{2 r}{m \lambda \sigma}} + L \sqrt{\frac{2 \rho}{m \sigma}}$, and observe that it verifies the condition of Theorem~\ref{eq:bousquet_local}.
Next, to obtain the upper bound on $r_m^\star$, we solve $L \sqrt{\frac{2 r}{m \lambda \sigma} + \frac{2 \rho}{m \sigma}} \leq r$
and get that
$
r_m^\star \leq \frac{L (L + 1)}{m \lambda \sigma} + L \sqrt{\frac{2 \rho}{m \sigma}}.
$
As in Theorem~\ref{thm:htl_gen_bound}, we also get that $\Risk(h) \leq \Rsrc$.
Plugging $r_m^\star$ and the bound on $\Risk(h)$ into Theorem~\ref{eq:bousquet_local}, we have the statement.
\qed\end{proof}